\documentclass{article}

 \usepackage[preprint]{nips_2018}

\usepackage[utf8]{inputenc} 
\usepackage[T1]{fontenc}    
\usepackage{hyperref}       
\usepackage{url}            
\usepackage{booktabs}       
\usepackage{amsfonts}       
\usepackage{nicefrac}       
\usepackage{microtype}      

\usepackage{amsmath}
\usepackage{amsthm}
\usepackage{algpseudocode,algorithm,algorithmicx}
\usepackage{graphicx}
\graphicspath{ {images/} }
\usepackage[usenames, dvipsnames]{color}

\DeclareMathOperator*{\argmax}{arg\,max}

\definecolor{myorange}{RGB}{246, 164, 16}
\definecolor{mygreen}{RGB}{1, 100, 3}

\title{GANs beyond divergence minimization}

\author{
  Alexia Jolicoeur-Martineau \\
Lady Davis Institute\\
Montreal, Canada\\
\texttt{alexia.jolicoeur-martineau@mail.mcgill.ca} \\
}

\begin{document}

\maketitle

\begin{abstract}
	Generative adversarial networks (GANs) can be interpreted as an adversarial game between two players, a discriminator $D$ and a generator $G$, in which $D$ learns to classify real from fake data and $G$ learns to generate realistic data by "fooling" $D$ into thinking that fake data is actually real data. Currently, a dominating view is that $G$ actually learns by minimizing a divergence given that the general objective function is a divergence when $D$ is optimal. However, this view has been challenged due to inconsistencies between theory and practice. In this paper, we discuss of the properties associated with most loss functions for $G$ (e.g., saturating/non-saturating $f$-GAN, LSGAN, WGAN, etc.). We show that these loss functions are not divergences and do not have the same equilibrium as expected of divergences. This suggests that $G$ does not need to minimize the same objective function as $D$ maximize, nor maximize the objective of $D$ after swapping real data with fake data (non-saturating GAN) but can instead use a wide range of possible loss functions to learn to generate realistic data. We define GANs through two separate and independent $D$ maximization and $G$ minimization steps. We generalize the generator step to four new classes of loss functions, most of which are actual divergences (while traditional $G$ loss functions are not). We test a wide variety of loss functions from these four classes on a synthetic dataset and on CIFAR-10. We observe that most loss functions converge well and provide comparable data generation quality to non-saturating GAN, LSGAN, and WGAN-GP generator loss functions, whether we use divergences or non-divergences. These results suggest that GANs do not conform well to the divergence minimization theory and form a much broader range of models than previously assumed.
\end{abstract}

\section{Introduction}

Generative adversarial networks (GANs) form a class of generative models that is most famously known for generating state-of-the-art photo-realistic images \citep{zhang2017stackgan} \citep{karras2017progressive}.  Note that we refer to the original version of GAN \citep{GAN} as ``standard GAN'' and to all variants of generative adversarial networks that work in a similar fashion as ``GANs''.

GANs consist in training two neural networks, a discriminator $D$ and a generator $G$, that work in competition so that $G$ can learn to generate fake data that appears to be genuine. $D$ is trained to differentiate real from fake data, which is done by classifying real from fake data ($D(x_{real}) \to y_{real}$ and $D(x_{fake}) \to y_{fake}$) or, more generally, by maximizing the expectation of $f(D(x_{real}))$ and $g(D(x_{fake}))$, where $f$ is generally monotone increasing and $g$ is generally monotone decreasing. $G$ takes as input a random number from $\mathbb{P}_z$, generally a multivariate normal distributed centered at $0$ with variance 1, and output a randomly generated fake data.

GANs are generally interpreted from two differing point-of-views: (1) adversarial game and (2) divergence minimization. In the former, $G$ is trained to maximize the same objective function as $D$ but swapping real data with fake data, thus, intuitively, $G$ tries to fool $D$ into thinking fake data is real data. In the latter, $G$ is trained to minimize the same objective function that was previously maximized by $D$. Given that the loss of $D$ is generally an approximation of a divergence (if $D$ is optimal, it is exactly equal to the divergence), $G$ is assumed to be minimizing a divergence.

We start by presenting these two differing views in detail. Then, we explain why minimizing the loss of $G$ cannot generally be interpreted as minimizing a divergence. We present four general forms of loss functions that can be used to train the generator, some of which are shown to be divergence. Through experiments, we show that most of these loss functions converge well. Finally, we discuss the implications of these results. 
 
The main contributions of this paper are showing that:
\begin{enumerate}
	\item In most GANs, the loss of $G$ is not a divergence even when $D$ is optimal
	\item $G$ does not directly minimize the divergence assumed by the objective function of $D$
	\item The loss of $G$ does not need to match the objective of $D$ and can be
	\begin{itemize}
		\item matching the mean or individual discriminator output of the fake data to the real labels $\left( \mathbb{E}[D(x_{fake})] \to y_{real} \text{ or } D(x_{fake}) \to y_{real} \right)$
		\item  matching the mean or individual discriminator output of the fake data to the classification threshold $\left(\mathbb{E}[D(x_{fake})] \to y_{mid} \text{ or } D(x_{fake}) \to y_{mid} \right)$
		\item mean matching the discriminator output of the fake data to the discriminator output of the real data $\left(\mathbb{E}[D(x_{fake})] \to \mathbb{E}[D(x_{real})] \right)$
	\end{itemize}
	\item Using actual divergences for the loss functions of $G$ does not provide any benefit
\end{enumerate}

\section{GANs interpretations}

\subsection{First interpretation: adversarial game}

In the first interpretation, GANs are understood as an adversarial game (\citet{GAN}; \citet{DRAGAN}; \citet{heusel2017gans}) between $D$ and $G$ in which $D$ tries to classify which data is real or fake while $G$ tries to fools $D$ into thinking the fake data it generates is actually real data. To fool $D$, the generator maximize the objective function of $D$ after swapping real data with fake data. \citet{GAN} showed that maximizing the objective of $D$ after swapping data (non-saturating GAN) works much better in practice than directly minimizing the objective of $D$ (saturating GAN). Non-saturating GAN can be represented  mathematically as the following two steps:
\begin{equation}
\begin{aligned}
\bullet\qquad & \max_{D:X \rightarrow [0,1]} \mathbb{E}_{x \sim \mathbb{P}}[\log(D(x))] + \mathbb{E}_{z \sim \mathbb{P}_z}[\log(1 - D(G(z)))] \\
\bullet\qquad & \max_{G:Z \rightarrow X} \mathbb{E}_{z \sim \mathbb{P}_z}[\log(D(G(z)))] + \mathbb{E}_{x \sim \mathbb{P}}[\log(1 - D(x))],
\end{aligned}
\end{equation}
where $\mathbb{P}$ is the distribution of the real data on domain $X$ and $Z$ is the domain of $\mathbb{P}_z$. Note that we generally denote the distribution of the fake data formed by $G(z)$ as $\mathbb{Q}$.

This can be generalized in the following matter:
\begin{equation}
\begin{aligned}
\bullet\qquad & \max_{D:X \rightarrow domf} \mathbb{E}_{x \sim \mathbb{P}}[f(D(x))] + \mathbb{E}_{z \sim \mathbb{P}_z}[g(D(G(z)))] \\
\bullet\qquad & \max_{G:Z \rightarrow X} \mathbb{E}_{z \sim \mathbb{P}_z}[f(D(G(z)))] + \mathbb{E}_{x \sim \mathbb{P}}[g(D(x))],
\end{aligned}
\end{equation}
where $f$ and $g$ are scalar-to-scalar functions chosen so that $D$ is a discriminator that predicts the likelihood of the data being real; generally $f$ is monotone increasing and $g$ is monotone decreasing.

Many GANs fit into this category, the main difference being the objective function chosen. Some examples are: non-saturating GAN \citep{GAN} which uses cross-entropy, LSGAN \citep{LSGAN} which uses least squares and MAGAN \citep{MAGAN} which uses hinge loss.

\subsection{Second interpretation: divergence minimization}

In the second interpretation, GANs are understood as divergence minimization (\citet{F-GAN}; \citet{WGAN}, \citet{mroueh2017fisher}; \citet{li2017mmd}; \citet{bellemare2017cramer}; \citet{mroueh2017sobolev}). Divergences are a weak form of distance between two probability distributions with the two following properties: non-negative and equal to zero if and only if the distributions are equal. Some well-known divergences are the Kullback–Leibler distance (KL), the Jensen-Shannon distance (JSD), and the Wasserstein distance. To prevent confusion between the discriminator and the divergence, we denote divergences between two distributions $\mathbb{P}$ and $\mathbb{Q}$ as $D_{*}(\mathbb{P}||\mathbb{Q})$. Some divergences like JSD and Wasserstein have symmetry so $D_{*}(\mathbb{P}||\mathbb{Q})= D_{*}(\mathbb{Q}||\mathbb{P})$, but otherwise this is not the case.

From this perspective, one tries to find the generator (with parameters $\theta$) that minimizes a divergence between real and fake data: 
\[\min_{\theta} D_{*}(\mathbb{P} || \mathbb{Q}_{\theta}).\]
However, commonly used divergences, such as $f$-divergences (a class of divergences for which KL and JSD are special cases), are difficult to minimize given that they require knowing the probability density functions of the real and fake data, $p(x)$ and $q(x|\theta)$ respectively. In practice, we do not know the probability distributions of the real data or the fake data. 

Traditionally, one approximates $\mathbb{P}$ as $\hat{\mathbb{P}}$ using the empirical distribution, i.e., a discrete distribution where each data sample has probability $1/n$. In this case, it can be shown \citep{GANTutorial} that minimizing the KL-divergence is equivalent to maximizing the log-likelihood (one of the most popular approaches in machine learning and statistics):
\[ \min_{\theta} D_{KL}(\hat{\mathbb{P}} || \mathbb{Q}_{\theta}) = \max_{\theta} \log \mathcal{L}(\theta; x). \]

In the divergence point-of-view, GANs try to minimize an equivalent ``dual'' parametrization of a divergence that does not require knowing the probability density functions. Common divergences can be represented with respect to $D(x)$ rather than $p(x)$ and $q(x|\theta)$ .

For saturating GAN, it can be shown \citep{GAN} that JSD is equal to an affine function of the minimum cross-entropy:
\begin{equation}
JSD(\mathbb{P} || \mathbb{Q}) = \frac{1}{2} \left(log(4) + \max_{D:X \rightarrow [0,1]} \mathbb{E}_{x \sim \mathbb{P}}[\log(D(x))] + \mathbb{E}_{x \sim \mathbb{Q}}[1 - \log(D(x))] \right).
\end{equation}
\citet{F-GAN} generalized this concept to the class of $f$-divergences and showed that:
\begin{equation}
D_{f}(\mathbb{P} || \mathbb{Q}) = \sup_{D:X \rightarrow \text{dom}_{f^*}} \mathbb{E}_{x \sim \mathbb{P}}[D(x)] - \mathbb{E}_{x \sim \mathbb{Q}}[f^*(D(x))],
\end{equation}
where $f^*$ is the convex conjugate of $f$ and $f$ is the function that defines the $f$-divergence used (e.g., $f(u) = u\log u$ leads to KL and $f(u) = u\log u - (u+1)\log(u+1)$ leads to saturating GAN).

Thus, on a very general level, GANs can be formulated in the following matter:
\begin{equation}
\min_{G:Z \rightarrow X} D_{*}(\mathbb{P} || \mathbb{Q}_G) = \min_{G:Z \rightarrow X} \max_{D:X \rightarrow domf} \mathbb{E}_{x \sim \mathbb{P}}[f(D(x))] + \mathbb{E}_{z \sim \mathbb{P}_z}[g(D(G(z)))].
\end{equation}
Note that, in practice, we let $G$ and $D$ be neural networks and we optimize for their respective parameters $\theta$ and $w$.

Most GANs fit into this category, some examples are: saturating GAN \citep{GAN} which minimize the JSD, saturating $f$-GAN \citep{F-GAN} which minimize $f$-divergences, and Wasserstein GAN (WGAN) \citep{WGAN} which minimize the Wasserstein distance.

There are a few issues regarding this interpretation of GANs. Firstly, standard GAN and $f$-GANs actually converge better when solving the optimization problem of equation (2) rather than equation (5) \citep{F-GAN}. Secondly, non-saturating GANs are able to learn the distribution of the real data even when directly minimizing the JSD fails \citep{ManyPaths} because the gradient of the divergence is actually constant or infinite \citep{WGAN}. 

With the exception of standard GAN and $f$-GANs, GANs that follow the divergence minimization interpretation are generally based on integral probability metrics (IPMs) \citep{muller1997integral}:
\[
IPM_{F} (\mathbb{P} || \mathbb{Q}) = \sup_{D \in F} \mathbb{E}_{x \sim \mathbb{P}}[D(x)] - \mathbb{E}_{x \sim \mathbb{Q}}[D(x)],
\]
where $F$ is a class of functions chosen to prevent the supremum from being infinite. See \citet{mroueh2017sobolev} for a summary of the various IPMs used in the literature. Importantly, IPM-based GANs can still be understood as following equation (2) considering that $\min_{\theta} -\mathbb{E}_{z \sim \mathbb{P}_z}[D(G_{\theta}(z))]$ is the same as $\max_{\theta} \mathbb{E}_{z \sim \mathbb{P}_z}[D(G_{\theta}(z))]$. Therefore, the fact these GANs work without swapping real data with fake data doesn't necessarily disprove the possibility of the adversarial game interpretation being actually correct.

In the following section, we show what loss function $G$ actually minimizes in GANs and we present four new types of loss functions, most of which are divergences, that $G$ can minimize instead of the non-saturating/saturating loss function generally assumed (equation 2 or 5).

\section{Loss of the generator}

\subsection{What $G$ is actually minimizing}

If we concentrate entirely on the generator step, for the vast majority of GANs, the stochastic gradient descent (SGD) step can be formulated as:
\begin{equation}
\min_{\theta} \mathbb{E}_{z \sim \mathbb{P}_z}[h(D(G_{\theta}(z)))],
\end{equation}
where $h$ is a scalar-to-scalar function (generally monotone decreasing).
Given that this is the only part of the equation that is used by SGD, this is effectively the loss function minimized rather than the divergence envisioned.

For saturating GAN, we have that $h(D(x))=log(1-D(x)) \in [-\infty,0]$. The loss function is not lower bounded since we have that $h(D(x)) \to -\infty$ as $D(x) \to 1$. This is problematic because there is no minimum and if $D(x) = 1$ for a single real sample, the expectation equals to $-\infty$ and the infinimum is reached. On the other hand, for non-saturating GAN, the loss is bounded as we have that $h(D(x))=-log(D(x)) \in [0,\infty]$. Therefore, minimizing this loss, we have that $h(D(x)) \to 0$ as $D(x) \to 1$.

These observations can be generalized to a broader class of GANs with $f$-GANs (equation 4). For the saturating loss, $h(D(x)) = -f^*(D(x))$, where $f^*$ is a convex function \citep{F-GAN}. Given the convexity of $f^*$, the maximum of $f^*$ (or equivalently the minimum of $-f^*$) is reached when $D(x) \to \min_x D(x)$ or $D(x) \to \max_x D(x)$. For most common divergences, we have that $f^*$ is monotone increasing; out of all the $f$-divergences presented by \citet{F-GAN} (KL, reverse KL, Pearson/Neyman $\chi^2$, Squared Hellinger, Jeffrey, JSD), only Pearson $\chi^2$ is not monotone increasing. Thus, in all common $f$-divergences, except Pearson $\chi^2$, we have that the optimum is reached when $D(x) \to \max_x D(x)$. The non-saturating loss is $h(D(x))=-D(x)$, thus it is minimized also when $D(x) \to \max_x D(x)$. The loss function $h(D(x))$ is lower bounded (has a minimum) if $D(x)$ is upper bounded for the non-saturating loss function and if $f^*$ is upper bounded for the saturating loss function.

In the commonly assumed parametrization of LSGAN \citep{LSGAN}, we have that $h(D(x))=(D(x)-1)^2 \in [0,\infty]$. Therefore, minimizing this loss, we have that $h(D(x)) \to 0$ as $D(x) \to 1$, just as saturating GAN.

WGAN-GP \citep{WGAN-GP} is a very popular variant of WGAN which impose a constraint so that the magnitude of the gradient of $D$ for any point interpolated between real and fake data must be close to 1 (i.e., $| | \nabla_{\hat{x}} D(\hat{x}) | | \approx 1$, where $\hat{x} = \alpha x_{real} + (1-\alpha) x_{fake}$ and $\alpha \in [0,1]$). In WGAN-GP, we have that $h(D(x))=-D(x) \in [-\infty,\infty]$. Although this loss is not lower bounded, the gradient penalty forces the gradient of $D$ to be around 1, thus the step taken by SGD cannot be too large. This suggests that the gradient penalty may be useful because it limits how much $G$ can change in one minimization step.

By definition, a zero divergence should arise when an optimal $D$ is not able to discriminate real from fake data, i.e., when $D(x)= y_{mid}$ for all $x$, where $y_{mid}$ is the classification threshold ($\frac{1}{2}$ in standard GAN, $f(\frac{p(x)}{q(x|\theta)}) = f(1)$ in $f$-GAN) or when $\mathbb{E}_{x \sim \mathbb{P}}[D(x)] = \mathbb{E}_{x \sim \mathbb{Q}_{\theta}}[D(x)]$ given $D \in F$ in IPM-based GANs.  However, as mentioned above, to reach the optimum, $G$ generally attempt to make $D(x_{fake})\to \max_x D(x)$ or $D(x_{fake})\to y_{real}$, thus pushing $D(x_{fake})$ as far as possible away from  $y_{mid}$. This means that after training $G$, the discriminator $D$ cannot be optimal anymore. The divergence can even become negative, which is impossible by the definition of a divergence. This is something that we observe in practice; the objective function of the discriminator becomes positive (or large) after training $D$ and negative (or very small) after training $G$, thus cycling between an approximated divergence and a non-divergence. 

This means that in most currently used GANs ($f$-GAN and all GANs present in the large-scale study by \citet{lucic2017gans} \footnote{With the exception of BEGAN \citep{berthelot2017began} which is very unusual and cannot really be considered a GAN given that D is a auto-encoder rather than a discriminator function.}), the loss function of $G$ is not a divergence. This suggests that $G$ is not directly minimizing the divergence, but only indirectly by updating its weights so that $D(x_{fake}) \to \max_x D(x)$ or $D(x_{fake})\to y_{real}$ which result in $G$ generating more realistic data.

One way to reconcile these observations with divergence minimization is to interpret GANs as effectively minimizing a divergence without the assumption of an optimal $D$ (overshooting the goal) and re-estimating the divergence by training $D$ to optimality in the next steps. This only applies to GANs such as WGAN/WGAN-GP which train $D$ for a large number of iterations before training $G$. In this point-of-view, GANs act comparably to projected gradient descent, i.e., we take a step into the gradient direction and then project back into the feasible set (with the constraint that we impose, an optimal $D$). However, the difference here is that, with constraint optimization, we can make sure that the constraint is respected entirely, but with GANs, we can never truly train $D$ to be optimal. Also with constraint optimization, it is possible for the loss to still respect the constraint after minimizing it without constraint, but with GANs, it is impossible for $D$ to be optimal after training $G$ for even a single step. 

Training $D$ to optimality before training $G$, to try to approximately minimize the divergence, does not necessarily lead to better results \citep{ManyPaths}. The current state-of-the-art in generation of human faces \citep{karras2017progressive} used WGAN-GP with only one discriminator update. Given that training $D$ multiple times before $G$ significantly increase the training time and does not necessarily lead to generated samples of better quality, trying to render GAN training analogous to divergence minimization may be unnecessary and overly constraining research/practice to a small subset of all possible GANs.

\subsection{Generalizing the $G$ step}

As shown above, the generator generally take a step so that $D(x_{fake})$ reach for $\max_x D(x)$ or $y_{real}$. For the loss of $G$ to be a divergence, one could think that we should have $D(x_{fake})$ instead reach for $D(x_{real})$ or $y_{mid}$. We generalize this idea to four general forms of loss functions for $G$:

Discriminator Matching (DM):
\begin{equation}
DM(\mathbb{P}, \mathbb{Q}_{\theta}| D) = \mathbb{E}_{\substack{x_{real} \sim \mathbb{P}  \\ x_{fake} \sim \mathbb{Q}_{\theta}}}[d(D(x_{real}), D(x_{fake}))],
\end{equation}
Label Matching (LM):
\begin{equation}
LM(\mathbb{P}, \mathbb{Q}_{\theta}| D) = \mathbb{E}_{x \sim \mathbb{P}}[d(D(x), \hat{y})] + \mathbb{E}_{x \sim \mathbb{Q}_{\theta}}[d(D(x),\hat{y})],
\end{equation}
Expectation Discriminator Matching (EDM):
\begin{equation}
EDM(\mathbb{P}, \mathbb{Q}_{\theta}| D) = d \left( \mathbb{E}_{x \sim \mathbb{P}}[D(x)], \mathbb{E}_{x \sim \mathbb{Q}_{\theta}}[D(x)] \right),
\end{equation}
Expectation Label Matching (ELM):
\begin{equation}
ELM(\mathbb{P}, \mathbb{Q}_{\theta}| D) = d \left( \mathbb{E}_{x \sim \mathbb{P}}[D(x)], \hat{y} \right) + d \left( \mathbb{E}_{x \sim \mathbb{Q}_{\theta}}[D(x)] \right),
\end{equation}
where $D$ is a discriminator trained in any way to differentiate real from fake data, $d$ is a distance function (ideally a metric) and $\hat{y} \in [y_{mid},y_{real}]$. See Algorithm 1 for how to train GANs using these loss functions. If $\hat{y}=y_{mid}$, we are striving for equilibrium ($\mathbb{P}=\mathbb{Q}_{\theta}$). If $\hat{y}=y_{real}$, we are overshooting, as usually done in GANs. Note that, although we presented $y_{real}$ as the label for real data, many GANs do not have labels (e.g., $f$-GANs and IPM-based GANs). When there is no label, assuming that $D$ is trained using equation 2 or 5, one can let $y_{real}$ be defined as $\argmax_{D(x)} f(D(x))$ (generally just $\max_x D(x)$) and $y_{fake}$ be defined as $\argmax_{D(x)} g(D(x))$ (generally just $\min_x D(x)$). If $y_{real} < \infty$, one can use LM or ELM with $\hat{y} = y_{real}$. Also note that certain GANs do not have a $y_{mid}$ (e.g., $D(x) = c$, for any constant $c$, lead to a Wassertein distance of 0), so these approaches cannot use LM or ELM with $\hat{y} = y_{mid}$.

\begin{algorithm}  
	\caption{General training algorithm for GANs
		\label{alg:1}}
	\begin{algorithmic}
		\Require{The number of $D$ iterations $n_{D}$ ($n_{D}=1$ unless one seek to train $D$ to optimality), batch size $m$, functions $f$ and $g$ which determine the objective function of the discriminator, distance function $d$, $L_G$ the type of loss function used for $G$ and $\hat{y} \in [y_{mid}, y_{real}]$.}
		\While{$\theta$ has not converged}
		\For{$t = 1, \dots, n_{D}$}
		\State Sample $\{x^{(i)}\}_{i=1}^m \sim \mathbb{P}$ 
		\State Sample $\{z^{(i)}\}_{i=1}^m \sim \mathbb{P}_z$
		\State Update $w$ using SGD by ascending with $\nabla_{w} \frac{1}{m} \sum_{i=1}^{m} \left[ f(D_w(x^{(i)})) + g(D_w(G_{\theta}(z^{(i)}))) \right]$
		\EndFor 
		\State Sample $\{z^{(i)}\}_{i=1}^m \sim \mathbb{P}_z$
		\If {$L_G = DM$}
		\State Sample $\{x^{(i)}\}_{i=1}^m \sim \mathbb{P}$
		\State Update $\theta$ using SGD by descending with $\nabla_{\theta} \frac{1}{m} \sum_{i=1}^{m} \left[ d(D_w(x^{(i)}), D_w(G_{\theta}(z^{(i)}))) \right]$
		\ElsIf  {$L_G = LM$}
		\State Update $\theta$ using SGD by descending with $\nabla_{\theta} \frac{1}{m} \sum_{i=1}^{m} \left[ d(D_w(G_{\theta}(z^{(i)})), \hat{y}) \right]$
		\ElsIf  {$L_G = EDM$}
		\State Sample $\{x^{(i)}\}_{i=1}^m \sim \mathbb{P}$
		\State Update $\theta$ using SGD by descending with $\nabla_{\theta} \left[d(\sum_{i=1}^{m} \frac{1}{m} D_w(x^{(i)}), \sum_{i=1}^{m} \frac{1}{m} D_w(G_{\theta}(z^{(i)}))) \right]$
		\ElsIf  {$L_G = ELM$}
		\State Update $\theta$ using SGD by descending with $\nabla_{\theta} \left[ d(\sum_{i=1}^{m} \frac{1}{m} D_w(G_{\theta}(z^{(i)})), \hat{y}) \right]$
		\EndIf
		\EndWhile
	\end{algorithmic}
\end{algorithm}

It can be shown that if $d$ is a positive-definitive function ($d(x,y)\geq 0$ and $d(x,y)=0 \iff x=y$), $\hat{y}=y_{mid}$ and $D$ is optimal, all of these loss functions are divergences (See Appendix for proof and details). DM and LM are more general as they only require optimality at equilibrium ($p(x)=q(x|\theta) \iff D(x)=y_{mid}$), while EDM and ELM need the usual assumption of optimality and that $\mathbb{P}$ and $\mathbb{Q}$ have the same support.

Note that although these loss functions are divergences under the conditions mentioned above, the assumption of an optimal $D$ is still problematic given that the discriminator $D$ will always lose optimality after minimizing any loss function of $G$. This is true for all GANs and this is because modifying $G$ cannot change $D(x_{real})$. The only exception is when $\mathbb{P}=\mathbb{Q}$, if the loss function of $G$ is a divergence, the loss will already be zero, thus it will not change $D$. This is neat theoretical property which traditional loss functions don't have because they are not divergences as they push $D(x_{fake})$ toward imbalance rather than equilibrium.

\section{Experiments}

We trained GANs on a synthetic dataset (infinite swiss roll dataset \citep{marsland2015machine}) and CIFAR-10 \citep{krizhevsky2009learning}. All experiments were ran in PyTorch \citep{pytorch} using the Adam optimizer \citep{Adam} with hyperparameters $\beta_1=0.50$ and $\beta_2=0.999$.

\subsection{Synthetic experiments}

As a first experiment, we trained GANs on the simple swiss roll infinite dataset in $\mathbb{R}^2$ from scikit-learn \citep{scikit-learn} (see Appendix for more details). We used three different objective functions for $D$, cross entropy (as in standard GAN), least squares (as in LSGAN) and WGAN-GP objective function. As previously observed, the two-sided penalty for WGAN-GP works poorly in the swiss roll dataset \citep{blog}, thus we used a one-sided penalty, a variant that the authors of WGAN-GP found to lead to similar results in their own experiments \citep{WGAN-GP}.

We experimented with a wide range of loss functions for $G$, including the saturating loss, non-saturating loss, LSGAN, WGAN, the absolute value of the log difference, the squared value of the log difference, the absolute difference, the squared difference, and the pseudo-Huber loss \citep{barron2017more}. For most $G$ loss functions, we tried all variants of DM, EDM, LM and ELM. 

Both neural networks consisted of three linear layers followed by leaky ReLU activation functions \citep{maas2013rectifier} and one final linear layer. The discriminator was also followed by a sigmoid function when using the  cross-entropy loss. Learning rates of $5\mathrm{e}{-5}$ were used for both $G$ and $D$. We trained the models for 5000 cycles (one cycle = all $D$ iterations and one $G$ iteration) with a batch size of 256. We used 10 discriminator updates per cycle with penalty of 10 for the WGAN-GP models.

We reported the average root mean squared difference between real/fake samples and their nearest fake/real sample (NNRMSE) using 1000 real and fake samples. This can be defined mathematically as:
$\sum_{i=1}^{1000} \text{Nearest}(x_{real}(i), \hat{Q}_{\theta}) + \text{Nearest}(x_{fake}(i), \hat{P})$,
where $\text{Nearest}(x,S)$ finds the nearest neighbor of $x$ from the set $S$, $\hat{P}$ is the set of 1k real samples, $\hat{Q}_{\theta}$ is the set of 1k fake samples. This is a simple measure that give us a very good indication of how well the generator converge to the true data distribution as it penalize both under-coverage and over-coverage (See Figure 1 from Appendix for more details). We report the median NNRMSE from five runs with seed 1, 2, 3, 4, 5 respectively.

\subsection{Real-data experiments}

As a second experiment, we trained GANs on the CIFAR-10 dataset \citep{krizhevsky2009learning}. We used the same objective functions for $D$ and $G$ as in the synthetic experiment. For these experiments, we used the original two-sided penalty for WGAN-GP \citep{WGAN-GP}. The neural networks were following the DCGAN architecture \citep{DCGAN} using batch normalization \citep{BatchNorm}.
 
Learning rates of $1\mathrm{e}{-4}$ were used for training both $G$ and $D$ with a batch size of 32. We used 5 discriminator updates per cycle with penalty of 10 for the models using the WGAN-GP objective function. To compare models, we reported the Inception score (IS) \citep{tricks} \citep{barratt2018note} (larger is better) and the Fréchet Inception Distance (FID) \citep{heusel2017gans} \citep{FIDgithub} (smaller is better). Note that most researchers calculate the IS and FID using TensorFlow \citep{tensorflow2015-whitepaper} implementations, therefore the values we report may be sightly different. Given our limited computing power (a single GPU), we only trained the models once for 25 epochs using seed 1. Although 25 epochs was not enough to reach optimality, it was enough to detect non-convergence. Our goal with these experiments was not to show that we could achieve the state-of-the-art but simply to compare different loss functions on equal ground and show that most of them work just as well as standard loss functions.

\section{Results and discussion}

\begin{table}
	\caption{Fréchet Inception Distance (FID) after 25 epochs on the CIFAR-10 dataset}
	\label{CIFAR10}
	\centering
	\begin{tabular}{ccccccc}
		\toprule
		& \multicolumn{6}{c}{GAN}                   \\
		&  \multicolumn{3}{c}{$\mathbb{E}[d(D(x_{f}),\hat{y})]$} & \multicolumn{3}{c}{$d(\mathbb{E}[D(x_f)],\mathbb{E}[\hat{y}])$} \\
		$d(x,y)$ & $D(x_{r})$ & $y_{mid}$ & $y_{real}$ & $D(x_{r})$ & $y_{mid}$ & $y_{real}$ \\
		\cmidrule(r{2pt}){1-1} \cmidrule(l{2pt}r{2pt}){2-4} \cmidrule(l{2pt}){5-7}
		$-\mathbb{E}[log(D(x_f))]$ [Non-saturating] & \multicolumn{6}{c}{57.66}    \\
		$\mathbb{E}[log(1-D(x_f))]$ [Saturating] & \multicolumn{6}{c}{90.46} \\
		$|log(x)-log(y)|$ & 59.33 & 63.97 & 61.97 & 61.15 & 61.75 & 59.67 \\
		$(log(x)-log(y))^2$ & 64.50 & 61.34 & 62.06 & 59.33 & 63.40 & 56.00       \\
		$|x-y|$ & 405.84 & 62.85 & 64.50 & 60.84 & 62.90 & 66.69      \\
		$(x-y)^2$ & 63.17 & 64.66 & 65.82 & 65.62 & 65.08 & 62.38      \\
		$\sqrt{((x-y)^2 + 1)} - 1$ & 64.77 & 64.70 & 63.93 & 63.04 & 65.97 & 65.82  \\
	\end{tabular}
	\begin{tabular}{ccccccc}
		\toprule
		\phantom{$-\mathbb{E}[log(D(x_f))]$ [Non-saturating]} & \multicolumn{6}{c}{LSGAN}                   \\
		&  \multicolumn{3}{c}{$\mathbb{E}[d(D(x_{f}),\hat{y})]$} & \multicolumn{3}{c}{$d(\mathbb{E}[D(x_f)],\mathbb{E}[\hat{y}])$} \\
		$d(x,y)$ & $D(x_{r})$ & $y_{mid}$ & $y_{real}$ & $D(x_{r})$ & $y_{mid}$ & $y_{real}$ \\
		\cmidrule(r{2pt}){1-1} \cmidrule(l{2pt}r{2pt}){2-4} \cmidrule(l{2pt}){5-7}
		$\mathbb{E}[(D(x_f)-1)^2]$ [LSGAN] & \multicolumn{6}{c}{61.64}    \\
		$|x-y|$ & 63.79 & 65.80 & 56.57 & 58.06  & 63.14 & 61.59     \\
		$(x-y)^2$ &  60.66 & 62.78 & 58.23 & 63.81 & 68.41 & 61.64      \\
		$\sqrt{((x-y)^2 + 1)} - 1$ & 58.83 & 58.14 & 59.27 & 59.76 & 62.49 & 63.42  \\
	\end{tabular}
	\begin{tabular}{ccccccc}
		\toprule
		\phantom{$-\mathbb{E}[log(D(x_f))]$ [Non-saturating]} & \multicolumn{6}{c}{WGAN-GP (two-sided penalty)}                   \\
		&  \multicolumn{3}{c}{$\mathbb{E}[d(D(x_{f}),\hat{y})]$} & \multicolumn{3}{c}{$d(\mathbb{E}[D(x_f)],\mathbb{E}[\hat{y}])$} \\
		$d(x,y)$ & $D(x_{r})$ & $y_{mid}$ & $y_{real}$ & $D(x_{r})$ & $y_{mid}$ & $y_{real}$ \\
		\cmidrule(r{2pt}){1-1} \cmidrule(l{2pt}r{2pt}){2-4} \cmidrule(l{2pt}){5-7}
		$-\mathbb{E}[D(x_f)]$  [WGAN-GP] & \multicolumn{6}{c}{55.18} \\
		$|x-y|$ & 498.30 & & & 55.76 & & \\
		$(x-y)^2$ & 446.26 & & & 300.95 & & \\
		$\sqrt{((x-y)^2 + 1)} - 1$ & 136.25 & & & 54.86 & &  \\
		\bottomrule
	\end{tabular}
\end{table}

FIDs of the CIFAR-10 experiments are shown in Table 1. See Appendix for results of the swiss roll experiments (Table 2) and IS of the CIFAR experiments (Table 3). Overall, most loss functions converged well, with the exception of DM and LM with $\hat{y}=y_{mid}$ in the swiss roll dataset. Importantly, no loss function performed much better than other loss functions in a wide range of scenarios, thus there was no overall best.

A priori we expected that the loss functions that are divergences without requiring the assumption of same support (DM and LM with $\hat{y}=y_{mid}$) would work best. However, these divergences performed badly in the swiss roll dataset, while all loss functions performed equally well on CIFAR-10. This provide further evidence that the generator does not improve by minimizing a divergence but simply by trying to increase $D(x_{fake})$. We suspect that striving for $D(x_{fake}) \to y_{mid}$ may sometime have more difficulty converging than $D(x_{fake}) \to y_{real}$ in the swiss roll dataset because it is not taking a strong enough step to prevent $D$ from dominating. 


\section{Conclusion and future work}

In summary, most GANs do not directly minimize a divergence and trying to make GANs analogous to divergence minimization does not confer any benefit. Instead of training $G$ using the saturating or non-saturating loss, one can instead train $G$ using a wide range of possible loss functions. What we have shown is just a very small set of all the possible loss functions that one could use and we did not attempt to determine if some of these loss functions could lead to better state-of-the-art results in data generation. This paper brings a greater level of customization to GANs which we hope will lead to more diversity in GANs research (enlarging the GAN zoo) and new ways to improve data generation quality.

In this paper, we focused solely on the generator step; however, the discriminator step is as important, if not more. Issues or limitations of the discriminator will affect how well any loss function of $G$ will perform given that $D$ and its gradient are fundamental to the gradient of the loss of $G$. For example, in standard GAN, there are perfect discriminators ($D$ such that $D(x_{real})=1$ and $D(x_{fake})=0$ for all $x_{real} \in \mathbb{P}$ and $x_{fake} \in \mathbb{Q}$) for which $\nabla_x D(x)$ is exactly zero under certain theoretical conditions \citep{GANTheorems}. Thus, by the chain rule, any loss function of $G$ will also be zero when $D$ is one of those perfect discriminators. In practice, we can never obtain a perfect discriminator and close to perfect data separation becomes less likely over time as the support of $\mathbb{P}$ and $\mathbb{Q}$ get closer to one another, but this still shows a major issue with standard GAN that cannot be resolved simply by changing the loss of $G$. Thus, understanding what makes a discriminator "good" remains paramount. We encourage research in this direction rather than solely focusing on finding a "good" divergence that has informative gradients (generally an IPM) since $G$ is not minimizing this divergence directly.

Our results also suggest that feature matching (FM) \citep{tricks} may be more than a trick, but instead, a specific case of GAN, as FM can be seen as a special case of applying an EDM to the intermediate layers.

\bibliographystyle{unsrtnat}

\section{Appendix}

\theoremstyle{definition}
\newtheorem{definition}{Definition}[section]
\newtheorem{theorem}{Theorem}[section]
\newtheorem{corollary}{Corollary}[theorem]
\newtheorem{lemma}[theorem]{Lemma}

\theoremstyle{definition}
\begin{definition}\label{1.0}
	A function $d:X \times X \to [0,\infty)$ is \textbf{positive definite} if it respects the following two conditions:
	\begin{align*}
		&d(x,y) \ge 0 \\ 
		&d(x,y) = 0 \iff x = y.
	\end{align*}
	If $x$ and $y$ are probability distributions, $d$ is called a \textbf{divergence}.
\end{definition}
\theoremstyle{definition}
\begin{definition}\label{1.1}
	A discriminator $D:X \rightarrow Y$, where $Y \subset \mathbb{R}$, is said to be \textbf{optimal at equilibrium} on the distributions $\mathbb{P}$ and $\mathbb{Q}$ (with domain $X$) if there exists a $y_{mid} \in \mathbb{R}$ such that $ p(x)=q(x) \iff D(x)=y_{mid}$
\end{definition}
\theoremstyle{definition}
\begin{definition}\label{1.2}
	A discriminator $D:X \rightarrow Y$, where $Y \subset \mathbb{R}$, is said to be \textbf{optimal} on the distributions $\mathbb{P}$ and $\mathbb{Q}$ (with domain $X$) if \\
	(1) $D$ is optimal at equality \\
	(2) $D(x) = y_{mid} + d_1(p(x),q(x))$ when $p(x) > q(x)$, where $d_1$ is a positive definite function. \\
	(2) $D(x) = y_{mid} - d_2(p(x),q(x))$ when $p(x) < q(x)$, where $d_2$ is a positive definite function.
\end{definition}
This is just a way to formalize the notion of what is an optimal discriminator without resorting to any objective function for $D$. These definitions are very general; however, they do not apply to some GANs (e.g., WGAN, since any Lipschitz $D$ will lead to a Wasserstein distance of 0 when $\mathbb{P}=\mathbb{Q}$; GAN-GP, since $D(x)=y_{mid}$ for all $x$ when $\mathbb{P}=\mathbb{Q}$ would mean that $\nabla_x D(x) = 0$, but we enforce the constraint that $|| \nabla_x D(x)|| \approx 1$). Note that with standard GAN and LSGAN (with default parameters), it has been shown \citep{GAN} \citep{LSGAN} that the optimal discriminator is $D(x) = \frac{p(x)}{p(x)+q(x)}$. Thus, in both GAN and LSGAN, $D$ is optimal by def \ref{1.2}.

\begin{theorem}
	Let $d:X \times X \to [0,\infty)$ be positive-definite, $\mathbb{P}$ and $\mathbb{Q}$ distributions on the domain $X$ and $D:X \rightarrow Y$, where $Y \subset \mathbb{R}$. If $D$ is optimal at equilibrium, we have that $DM(\mathbb{P}, \mathbb{Q}| D)$ is a divergence and $LM(\mathbb{P}, \mathbb{Q}| D)$ is a divergence if $\hat{y}=y_{mid}$. If $D$ is optimal and $\mathbb{P}$ and $\mathbb{Q}$ have the same support, i.e., supp($\mathbb{P}$) = supp($\mathbb{Q}$) = supp($\mathbb{P} \cup \mathbb{Q}$), we have that
	$EDM(\mathbb{P}, \mathbb{Q}| D)$ is a divergence and $ELM(\mathbb{P}, \mathbb{Q}| D)$ is a divergence if $\hat{y}=y_{mid}$.
\end{theorem}
\begin{proof} 
	$ $\newline \newline
	$d$ is positive-definite $\implies d(x,y) \ge 0 \implies$ DM, LM, EDM, ELM are always $\ge 0$.
	\begin{align*}
		\phantom{\iff} &DM(\mathbb{P}, \mathbb{Q}| D) = 0 \\
		\iff &\mathbb{E}_{\substack{x_1 \sim \mathbb{P}  \\ x_2 \sim \mathbb{Q}}}[d(D(x_1), D(x_2))] = 0  \\
		\iff &d(D(x_1), D(x_2)) = 0 \quad \forall x_1 \in \text{supp}(\mathbb{P}), x_2 \in \text{supp}(\mathbb{Q}) \quad &&\text{since } d(D(x_1), D(x_2)) \ge 0 \\
		\iff &D(x_1) = D(x_2) \quad \forall x_1 \in \text{supp}(\mathbb{P}), x_2 \in \text{supp}(\mathbb{Q}) &&\text{since } d \text{ is positive-definite} \\
		\iff &D(x) = y_{mid} \quad \forall x \in \text{supp}(\mathbb{P} \cup \mathbb{Q}) &&\text{otherwise $D$ wouldn't be optimal} \\
		\iff &p(x) = q(x)\quad \forall x \in \text{supp}(\mathbb{P} \cup \mathbb{Q}) &&\text{since $D$ is optimal} \\
		\iff &\mathbb{P} = \mathbb{Q}
	\end{align*}
	Thus, DM is a divergence.
	\begin{align*}
		\phantom{\iff} &LM(\mathbb{P}, \mathbb{Q}| D) = 0 \\
		\iff &\mathbb{E}_{x_1 \sim \mathbb{P}}[d(D(x_1), \hat{y})] = 0 \text{ and } \mathbb{E}_{x_2 \sim \mathbb{Q}}[d(D(x_2),\hat{y})] = 0  \\
		\iff &d(D(x_1), \hat{y}) = 0 \text{ and } d(D(x_2),\hat{y}) = 0 \quad \forall x_1 \in \text{supp}(\mathbb{P}), x_2 \in \text{supp}(\mathbb{Q}) \\
		\iff &D(x_1) = \hat{y} \text{ and } D(x_2) = \hat{y} \quad \forall x_1 \in \text{supp}(\mathbb{P}), x_2 \in \text{supp}(\mathbb{Q}) \\
		\iff &D(x) = y_{mid} \quad \forall x \in \text{supp}(\mathbb{P} \cup \mathbb{Q}) \quad\quad \text{if } \hat{y}\neq y_{mid}, D\text{ cannot be optimal}  \\
		\iff &p(x) = q(x)\quad \forall x \in \text{supp}(\mathbb{P} \cup \mathbb{Q}) \\
		\iff &\mathbb{P} = \mathbb{Q}
	\end{align*}
	Thus, LM is only a divergence when $\hat{y} =  y_{mid}$. \\
	\begin{align*} 
	\phantom{\iff} &EDM(\mathbb{P}, \mathbb{Q}| D) = 0 \\
	\iff &d \left(\mathbb{E}_{x \sim \mathbb{P}}[D(x)], \mathbb{E}_{x \sim \mathbb{Q}}[D(x)] \right) = 0 \\
	\iff &\mathbb{E}_{x \sim \mathbb{P}}[D(x)] = \mathbb{E}_{x \sim \mathbb{Q}}[D(x)] \\
	\iff &\int_{supp(\mathbb{P})} D(x)p(x)\,d\mu(x) = 
	\int_{supp(\mathbb{Q})} D(x) q(x)\,d\mu(x) \\
	\iff &\int_{supp(\mathbb{P} \cup \mathbb{Q})} D(x)p(x)\,d\mu(x) = 
	\int_{supp(\mathbb{P} \cup \mathbb{Q})} D(x) q(x)\,d\mu(x)\\                                                                           
	\iff &\int_{supp(\mathbb{P} \cup \mathbb{Q}))} D(x)(p(x)-q(x))\,d\mu(x) = 0 \\
	& \text{Assuming that }p(x) \neq q(x) \text{ for some $x$, we have that}  \\
	& \int_{\substack{x \text{ s.t. } \\ p(x) > q(x)}} D(x)(p(x)-q(x))\,d\mu(x) + \int_{\substack{x \text{ s.t. } \\ p(x) < q(x)}} D(x)(p(x)-q(x))\,d\mu(x) = 0 \\
	& \int_{\substack{x \text{ s.t. } \\ p(x) > q(x)}} D(x)(p(x)-q(x))\,d\mu(x) = \int_{\substack{x \text{ s.t. } \\ p(x) < q(x)}} D(x)(q(x)-p(x))\,d\mu(x) \\
	& \text{On the left integral, we have that } D(x) > y_{mid} \text{ and } p(x)-q(x) \in (0,1) \\
	& \text{On the right integral, we have that } D(x) < y_{mid} \text{ and } p(x)-q(x) \in (0,1) \\
	& y_{mid} < \int_{\substack{x \text{ s.t. } \\ p(x) > q(x)}} D(x)(p(x)-q(x))\,d\mu(x) = \int_{\substack{x \text{ s.t. } \\ p(x) < q(x)}} D(x)(q(x)-p(x))\,d\mu(x) < y_{mid} \\
	& \text{This is impossible, thus we must have that } p(x)=q(x) \quad \forall x \\
	\iff &p(x) = q(x) \quad \forall x \in supp(\mathbb{P} \cup \mathbb{Q}) \\
	\iff &\mathbb{P} = \mathbb{Q}
	\end{align*}
	Thus EDM is a divergence when the two distributions have the same support.
	\begin{align*} 
	\phantom{\Longrightarrow} &\mathbb{P} = \mathbb{Q} \\ 
	\Longrightarrow &p(x) = q(x)\quad \forall x \in supp(\mathbb{P} \cup \mathbb{Q}) \\
	\Longrightarrow &D(x) = y_{mid} \quad \forall x \in supp(\mathbb{P} \cup \mathbb{Q}) \\
	\Longrightarrow &\mathbb{E}_{x_1 \sim \mathbb{P}}[D(x_1)] = \mathbb{E}_{x_2 \sim \mathbb{Q}}[D(x_2)] = y_{mid} \\
	\Longrightarrow &d(\mathbb{E}_{x_1 \sim \mathbb{P}}[D(x_1)], \hat{y})] = 0 \text{ and } d([\mathbb{E}_{x_2 \sim \mathbb{Q}}D(x_2)],\hat{y}) = 0 \quad\quad \text{only true if } \hat{y}=y_{mid} \\
	\Longrightarrow &ELM(\mathbb{P}, \mathbb{Q}| D) = 0
	\end{align*}
	Thus, just as LM, we need $\hat{y}=y_{mid}$ for ELM to possibly be a divergence.
	\begin{align*} 
	\phantom{\Longrightarrow} &ELM(\mathbb{P}, \mathbb{Q}| D) = 0 \\ \Longrightarrow &\mathbb{E}_{x \sim \mathbb{P}}[D(x)] = \mathbb{E}_{x \sim \mathbb{Q}}[D(x)] = y_{mid} \\
	& \text{(Follow same arguments as proof for EDM)} \\
	\Longrightarrow & \mathbb{P} = \mathbb{Q}
	\end{align*}
	Thus, ELM is a divergence when the two distributions have the same support and $\hat{y}=y_{mid}$.
\end{proof}

\begin{table}
	\caption{Median of the average root mean squared difference between real/fake samples and their nearest fake/real sample (NNRMSE), using 1000 real and fake samples, after 5 runs on the infinite swiss-roll dataset in $\mathbb{R}^2$}
	\label{spiral}
	\centering
	\begin{tabular}{ccccccc}
		\toprule
		 & \multicolumn{6}{c}{GAN}                   \\
		&  \multicolumn{3}{c}{$\mathbb{E}[d(D(x_{f}),\hat{y})]$} & \multicolumn{3}{c}{$d(\mathbb{E}[D(x_f)],\mathbb{E}[\hat{y}])$} \\
		$d(x,y)$ & $D(x_{r})$ & $y_{mid}$ & $y_{real}$ & $D(x_{r})$ & $y_{mid}$ & $y_{real}$ \\
		\cmidrule(r{2pt}){1-1} \cmidrule(l{2pt}r{2pt}){2-4} \cmidrule(l{2pt}){5-7}
		$-\mathbb{E}[log(D(x_f))]$ [Non-saturating] & \multicolumn{6}{c}{.057}     \\
		$\mathbb{E}[log(1-D(x_f))]$ [Saturating] & \multicolumn{6}{c}{.052}     \\
		$|log(x)-log(y)|$ & 1.401 & .282 & .054 & .055 & .062 & .052     \\
		$(log(x)-log(y))^2$ & 1.451 & 1.556 & .058 & 1.511 & .052 & .054      \\
		$|x-y|$ & 1.488 & 1.565 & .050 & .053 & .060 & .050     \\
		$(x-y)^2$ & 1.708 & 1.518 & .056 & .059 & .056 & .051      \\
		$\sqrt{((x-y)^2 + 1)} - 1$ & 1.536 & 1.693 & .050 & .055 & .052 & .050  \\
	\end{tabular}
	\begin{tabular}{ccccccc}
		\toprule
		\phantom{$-\mathbb{E}[log(D(x_f))]$ [Non-saturating]} & \multicolumn{6}{c}{LSGAN}                   \\
		&  \multicolumn{3}{c}{$\mathbb{E}[d(D(x_{f}),\hat{y})]$} & \multicolumn{3}{c}{$d(\mathbb{E}[D(x_f)],\mathbb{E}[\hat{y}])$} \\
		$d(x,y)$ & $D(x_{r})$ & $y_{mid}$ & $y_{real}$ & $D(x_{r})$ & $y_{mid}$ & $y_{real}$ \\
		\cmidrule(r{2pt}){1-1} \cmidrule(l{2pt}r{2pt}){2-4} \cmidrule(l{2pt}){5-7}
		$\mathbb{E}[(D(x_f)-1)^2]$ [LSGAN] & \multicolumn{6}{c}{.052}     \\
		$|x-y|$ & .163 & .164  & .053 & .053 & .057 & .052     \\
		$(x-y)^2$ & .089 & .147 & .052 & .055 & .055 & .054      \\
		$\sqrt{((x-y)^2 + 1)} - 1$ & .088 & .153 & .052 & .069 & .053 & .050  \\
	\end{tabular}
	\begin{tabular}{ccccccc}
		\toprule
		\phantom{$-\mathbb{E}[log(D(x_f))]$ [Non-saturating]} & \multicolumn{6}{c}{WGAN-GP (one-sided penalty)}                   \\
		&  \multicolumn{3}{c}{$\mathbb{E}[d(D(x_{f}),\hat{y})]$} & \multicolumn{3}{c}{$d(\mathbb{E}[D(x_f)],\mathbb{E}[\hat{y}])$} \\
		$d(x,y)$ & $D(x_{r})$ & $y_{mid}$ & $y_{real}$ & $D(x_{r})$ & $y_{mid}$ & $y_{real}$ \\
		\cmidrule(r{2pt}){1-1} \cmidrule(l{2pt}r{2pt}){2-4} \cmidrule(l{2pt}){5-7}
		$-\mathbb{E}[D(x_f)]$  [WGAN-GP] & \multicolumn{6}{c}{.062}     \\
		$|x-y|$ & .291 & & & .062 & & \\
		$(x-y)^2$ & .297 & & & .070 & & \\
		$\sqrt{((x-y)^2 + 1)} - 1$ & .332 & & & .073 & &  \\
		\bottomrule
	\end{tabular}
\end{table}
\begin{table}
	\caption{Inception score (IS) after 25 epochs on the CIFAR-10 dataset}
	\label{CIFAR10}
	\centering
	\begin{tabular}{ccccccc}
		\toprule
		 & \multicolumn{6}{c}{GAN}                   \\
		&  \multicolumn{3}{c}{$\mathbb{E}[d(D(x_{f}),\hat{y})]$} & \multicolumn{3}{c}{$d(\mathbb{E}[D(x_f)],\mathbb{E}[\hat{y}])$} \\
		$d(x,y)$ & $D(x_{r})$ & $y_{mid}$ & $y_{real}$ & $D(x_{r})$ & $y_{mid}$ & $y_{real}$ \\
		\cmidrule(r{2pt}){1-1} \cmidrule(l{2pt}r{2pt}){2-4} \cmidrule(l{2pt}){5-7}
		$-\mathbb{E}[log(D(x_f))]$ [Non-saturating] & \multicolumn{6}{c}{3.36}    \\
		$\mathbb{E}[log(1-D(x_f))]$ [Saturating] & \multicolumn{6}{c}{2.09}     \\
		$|log(x)-log(y)|$ & 3.41 & 3.38 & 3.40 & 3.44 & 3.42 & 3.30 \\
		$(log(x)-log(y))^2$ & 3.43 & 3.36 & 3.27 & 3.32 & 3.50 & 3.43       \\
		$|x-y|$ & 1.03 & 3.36 & 3.38 & 3.38 & 3.26 & 3.38      \\
		$(x-y)^2$ & 3.34 & 3.36 & 3.34 & 3.40 & 3.37 & 3.35      \\
		$\sqrt{((x-y)^2 + 1)} - 1$ & 3.31 & 3.41 & 3.46 & 3.33 & 3.30 & 3.48  \\
	\end{tabular}
	\begin{tabular}{ccccccc}
		\toprule
		\phantom{$-\mathbb{E}[log(D(x_f))]$ [Non-saturating]} & \multicolumn{6}{c}{LSGAN}                   \\
		&  \multicolumn{3}{c}{$\mathbb{E}[d(D(x_{f}),\hat{y})]$} & \multicolumn{3}{c}{$d(\mathbb{E}[D(x_f)],\mathbb{E}[\hat{y}])$} \\
		$d(x,y)$ & $D(x_{r})$ & $y_{mid}$ & $y_{real}$ & $D(x_{r})$ & $y_{mid}$ & $y_{real}$ \\
		\cmidrule(r{2pt}){1-1} \cmidrule(l{2pt}r{2pt}){2-4} \cmidrule(l{2pt}){5-7}
		$\mathbb{E}[(D(x_f)-1)^2]$ [LSGAN] & \multicolumn{6}{c}{3.36}    \\
		$|x-y|$ & 3.28 & 3.42 & 3.27 & 3.28 & 3.30 & 3.33     \\
		$(x-y)^2$ & 3.30 & 3.25 & 3.51 & 3.28 & 3.21 & 3.36      \\
		$\sqrt{((x-y)^2 + 1)} - 1$ & 3.28 & 3.30 & 3.32 & 3.25 & 3.25 & 3.21  \\
	\end{tabular}
	\begin{tabular}{ccccccc}
		\toprule
		\phantom{$-\mathbb{E}[log(D(x_f))]$ [Non-saturating]} & \multicolumn{6}{c}{WGAN-GP (two-sided penalty)}                   \\
		&  \multicolumn{3}{c}{$\mathbb{E}[d(D(x_{f}),\hat{y})]$} & \multicolumn{3}{c}{$d(\mathbb{E}[D(x_f)],\mathbb{E}[\hat{y}])$} \\
		$d(x,y)$ & $D(x_{r})$ & $y_{mid}$ & $y_{real}$ & $D(x_{r})$ & $y_{mid}$ & $y_{real}$ \\
		\cmidrule(r{2pt}){1-1} \cmidrule(l{2pt}r{2pt}){2-4} \cmidrule(l{2pt}){5-7}
		$-\mathbb{E}[D(x_f)]$  [WGAN-GP] & \multicolumn{6}{c}{3.34} \\
		$|x-y|$ & 1.05 & & & 3.34 & & \\
		$(x-y)^2$ & 1.11 & & & 1.48 & & \\
		$\sqrt{((x-y)^2 + 1)} - 1$ & 1.75 & & & 3.26 & &  \\
		\bottomrule
	\end{tabular}
\end{table}

\begin{figure}
	\centering
	\includegraphics[scale=.50]{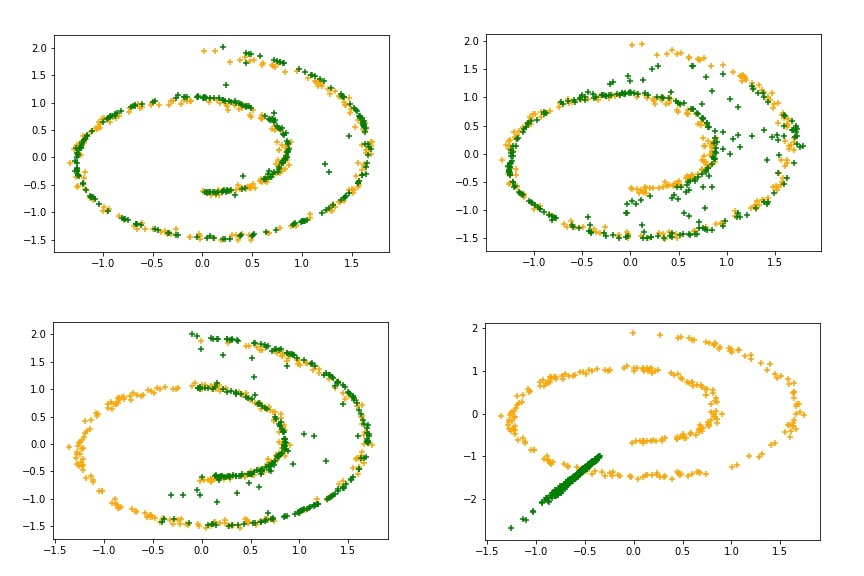}
	\caption{Real data from the infinite swiss roll dataset (\textcolor{myorange}{orange}) and fake data from the generator of a GAN (\textcolor{mygreen}{green}). The four figures show archetypal examples observed in practice with different median NNRMSEs. The \textbf{top left} example has low overall error (NNRMSE $\approx$ .05), the \textbf{top right} example has low fake-to-nearest-real error, but high real-to-nearest-fake error (NNRMSE $\approx$ .20), the \textbf{bottom left} example has low real-to-nearest-fake error, but high fake-to-nearest-real error (NNRMSE $\approx$ .50), and the \textbf{bottom right} example has high overall error (NNRMSE > 1).}
\end{figure}

\end{document}